\documentclass[letterpape]{article} % DO NOT CHANGE THIS
\usepackage{aaai2026}  % DO NOT CHANGE THIS
\nocopyright
\usepackage{times}  % DO NOT CHANGE THIS
\usepackage{helvet}  % DO NOT CHANGE THIS
\usepackage{courier}  % DO NOT CHANGE THIS
\usepackage[hyphens]{url}  % DO NOT CHANGE THIS
\usepackage{graphicx} % DO NOT CHANGE THIS
\usepackage{natbib}  % DO NOT CHANGE THIS AND DO NOT ADD ANY OPTIONS TO IT
\usepackage{caption} % DO NOT CHANGE THIS AND DO NOT ADD ANY OPTIONS TO IT
\usepackage{algorithm}
\usepackage{algorithmic}

\usepackage{newfloat}
\usepackage{listings}
\DeclareCaptionStyle{ruled}{labelfont=normalfont,labelsep=colon,strut=off} % DO NOT CHANGE THIS
\floatstyle{ruled}
\newfloat{listing}{tb}{lst}{}
\floatname{listing}{Listing}

\usepackage{amsmath, amssymb}
\usepackage{amsthm}
\newtheorem{theorem}{Theorem}

\newtheorem{proposition}{Proposition}  % 'Proposition' 환경 정의
\usepackage{subcaption}
\usepackage{booktabs} % table 서식
\usepackage{multirow}
\usepackage{makecell}
\title{Divergence-Based Similarity Function for Multi-View Contrastive Learning}

\author{
  Jaehyoung Jeon\textsuperscript{1}, 
  Cheolsu Lim\textsuperscript{1}, 
  Myungjoo Kang\textsuperscript{1,2} \\
  \textsuperscript{1}Department of Mathematics, Seoul National University, Seoul, Korea \\
  \textsuperscript{2}Research Institute of Mathematics, Seoul National University, Seoul, Korea \\
  \texttt{\{jan4021, sky3alfory, mkang\}@snu.ac.kr}, \texttt{jaehyoung.jeon95@gmail.com}
}

\usepackage{bibentry}
\begin{document}

\maketitle

\footnotetext{
Published in the proceedings of PAKDD 2026.
The Version of Record is available at
\url{https://doi.org/10.1007/978-981-92-1462-4_5}.
}

\begin{abstract}
Recent success in contrastive learning has sparked growing interest in more effectively leveraging multiple augmented views of data. While prior methods incorporate multiple views at the loss or feature level, they primarily capture pairwise relationships and fail to model the joint structure across all views. In this work, we propose a divergence-based similarity function (DSF) that explicitly captures the joint structure by representing each set of augmented views as a distribution and measuring similarity as the divergence between distributions. Extensive experiments demonstrate that DSF consistently improves performance across diverse tasks—including kNN classification, linear evaluation, transfer learning, and distribution shift—while also achieving greater efficiency than other multi-view methods. Furthermore, we establish a connection between DSF and cosine similarity, and demonstrate that, unlike cosine similarity, DSF operates effectively without the need for tuning a temperature hyperparameter. Code and pretrained models are available at \url{https://github.com/Jeon789/DSF}.
\end{abstract}

\section{Introduction}
\label{section:introduction}

Self-supervised learning (SSL) aims to learn meaningful representations without labels, often using contrastive learning (CL) frameworks such as SimCLR~\cite{chen2020simple}, MoCo~\cite{he2020momentum}, and DCL~\cite{yeh2022decoupled}. These methods encourage similarity between augmented views of the same data while separating different data, typically via cosine similarity on unit-normalized features.

However, most existing contrastive methods focus on pairwise relationships between two views. Recent efforts to leverage more than two views can be broadly categorized into two approaches:
\begin{itemize}
    \item \textbf{Loss-level approaches}: Averaging pairwise losses computed between all available combinations of views \cite{shidani2024poly}.
    \item \textbf{Feature-level approaches:} Aggregating features (e.g., via averaging~\cite{Fastsiam} or optimal transport~\cite{piran2024contrasting}) before applying the contrastive loss.
\end{itemize}

\begin{figure}[t]
    \centering
    \includegraphics[width=0.7\linewidth]{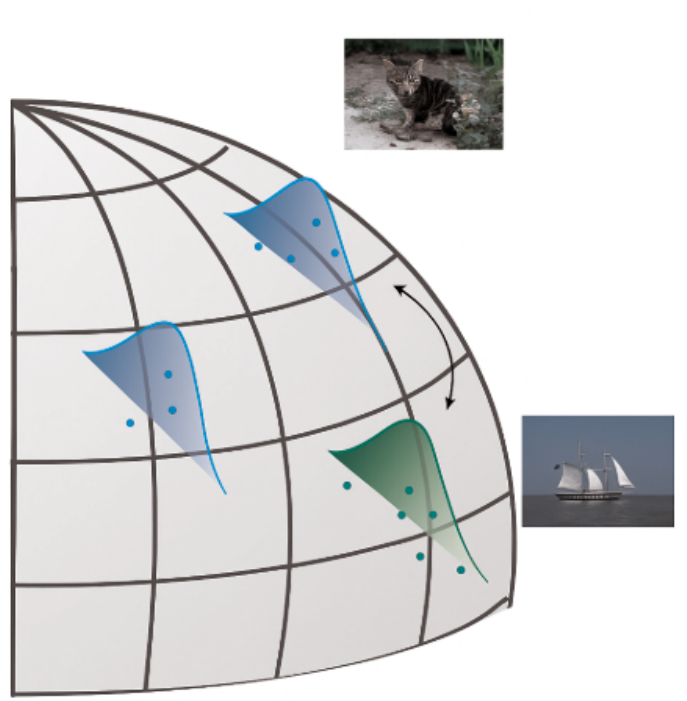}
    \caption{\textbf{Illustration of divergence-based similarity function}: given multiple augmented views of the image, their extracted feature vectors (colored dots) lie on the hypersphere and are used to construct probability distributions (colored regions) over the unit hypersphere.}
    \label{Fig:vMF_figure}
\end{figure}

\begin{figure*}[t!]
    \centering
    \includegraphics[width=\linewidth]{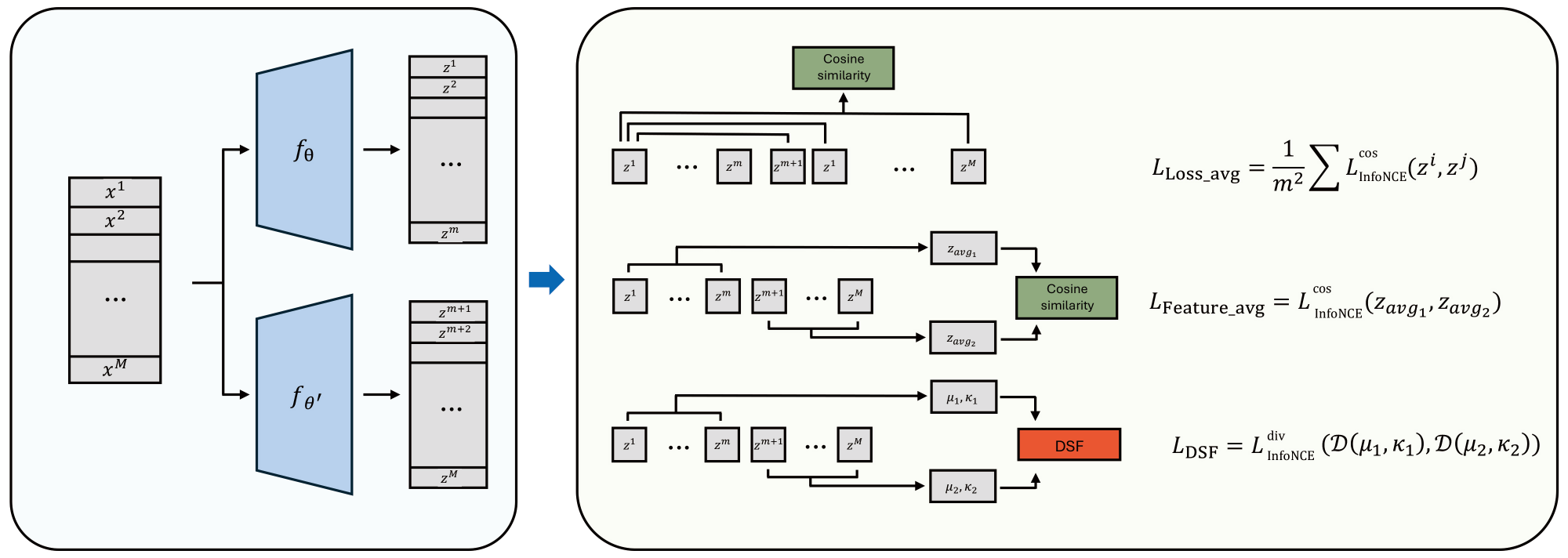}
    \caption{\textbf{Illustration of multi-view contrastive learning strategies.} Given an input $x$, we generate $M$ augmented views, resulting in $x^1, \cdots, x^M$. Base encoder outputs $M$ representations $z^1, \cdots, z^M$ from $M$ views. From top to bottom: (i) \textit{Loss\_avg}, which averages pairwise loss terms across $M$ views; (ii) \textit{Fea\_avg}, which computes a mean embedding over the $M$ views; and (iii) \textit{DSF}, our proposed method, which aggregates features from $M$ multiple views into distributions and computes a single contrastive loss between them.}
    \label{Fig:DSF_flow}
\end{figure*}

Although these multi-view extensions generalize two-view objectives, they still rely solely on pairwise relations and overlook the joint structural relationships among the full set of views.

To address this, we propose a novel similarity function that models the joint structure among all augmented views of data. We estimate a probability distribution on the unit hypersphere and compute similarity via negative divergence between distributions, see Figure \ref{Fig:vMF_figure}. We use the von Mises–Fisher (vMF) distribution, a widely used Gaussian-like distribution on the unit hypersphere \cite{mardia2009directional,fisher1993statistical,kitagawa2022mises,scott2021mises,du2024probabilistic}. vMF distribution is parameterized by mean direction $\mu$ and concentration $\kappa$, which respectively represent the central tendency and joint structure of features. Notably, in the two-view case, our method recovers cosine similarity in the InfoNCE loss~\cite{oord2018representation}, linking pairwise and multi-view contrastive frameworks.

Our main contributions are:
\begin{itemize}
    \item \textbf{Divergence-based multi-view Contrastive Learning.} We propose a Divergence-based Similarity Function (DSF) that captures joint structure across all views without decomposing into pairwise terms.
    \item \textbf{Establishing the relation to the original InfoNCE loss.} We show that, when restricted to two views under specific condition, our similarity recovers cosine similarity in InfoNCE loss, thus offering the unified view of pairwise and multi-view contrastive methods. 
    \item \textbf{Temperature-free.} We show the dependency of the temperature term in cosine similarity and show that, for DSF, temperature tuning is unnecessary.
    \item \textbf{Superior Experimental Results.} We validate the effectiveness of DSF through extensive experiments on kNN classification, linear evaluation, transfer learning, and distribution shift, as well as resource analysis in terms of GPU memory and training time, consistently outperforming both baseline and existing multi-view contrastive learning methods.
\end{itemize}

The remainder of this paper is structured as follows. Section~\ref{section:related_works} reviews related work on contrastive frameworks with both two-view ($M=2$) and multi-view ($M \ge 3$) settings. Section~\ref{section:prelim} provides an overview of CL and its multi-view extensions. Section~\ref{section:divergence-based_similarity_function} presents our proposed formulation. Section~\ref{section:experiments} describes the experimental setup and results. Finally, Section~\ref{section:conclusion} concludes the paper and discusses potential directions for future research.

%%%%%%%%%%%%%%%%%%%%%%%%%%%%%%%%%%%%%%%%%%%%%%%%%%%%%%%%%
\section{Related Works}
\label{section:related_works}

\begin{figure*}[!t]
    \centering
    \begin{subfigure}[t]{0.48\textwidth}
        \centering
        \includegraphics[width=\linewidth]{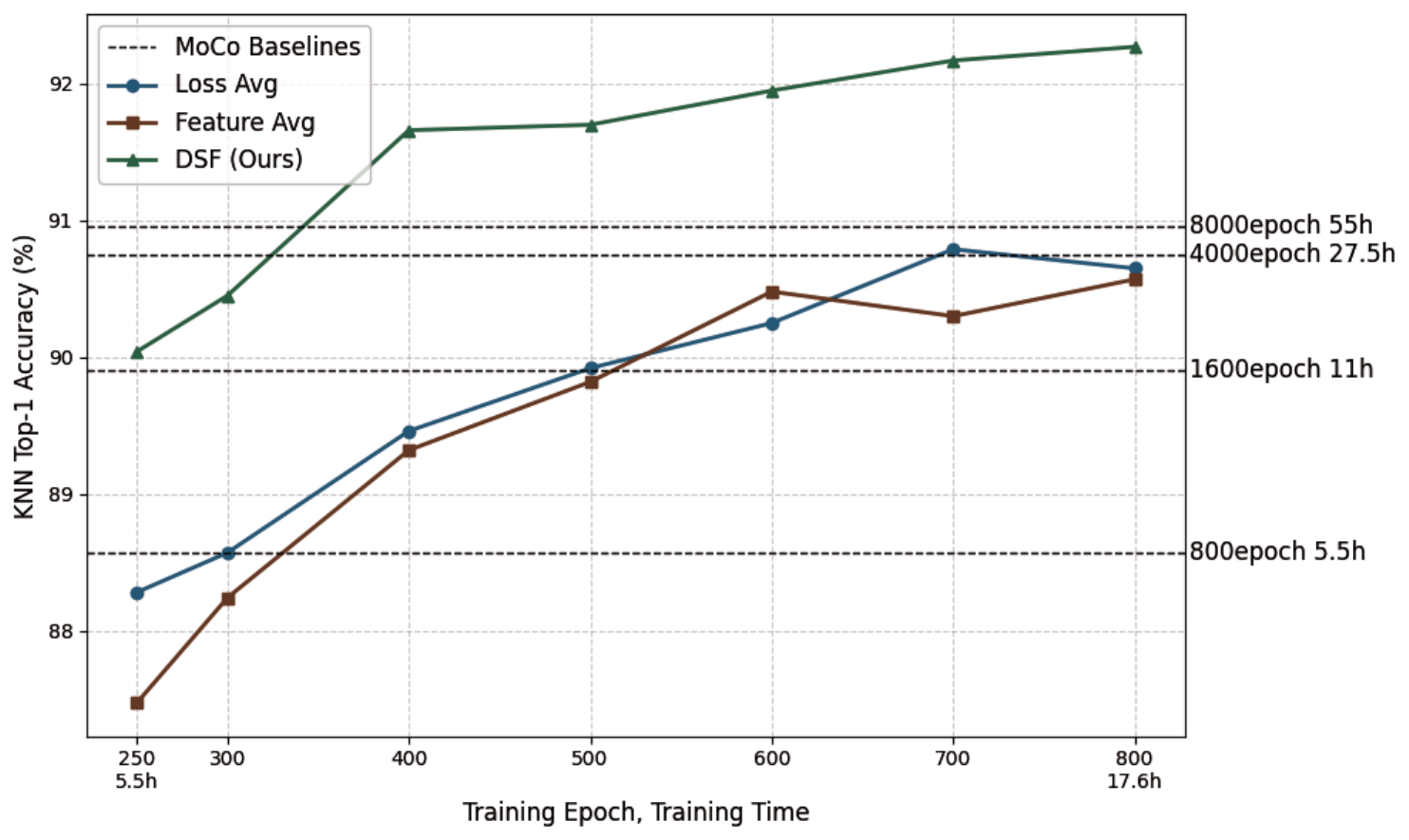}
        \caption{kNN Top-1 Accuracy with respect to training time.}
        \label{fig:knn_accuracy}
    \end{subfigure}
    \hfill
    \begin{subfigure}[t]{0.48\textwidth}
        \centering
        \includegraphics[width=\linewidth]{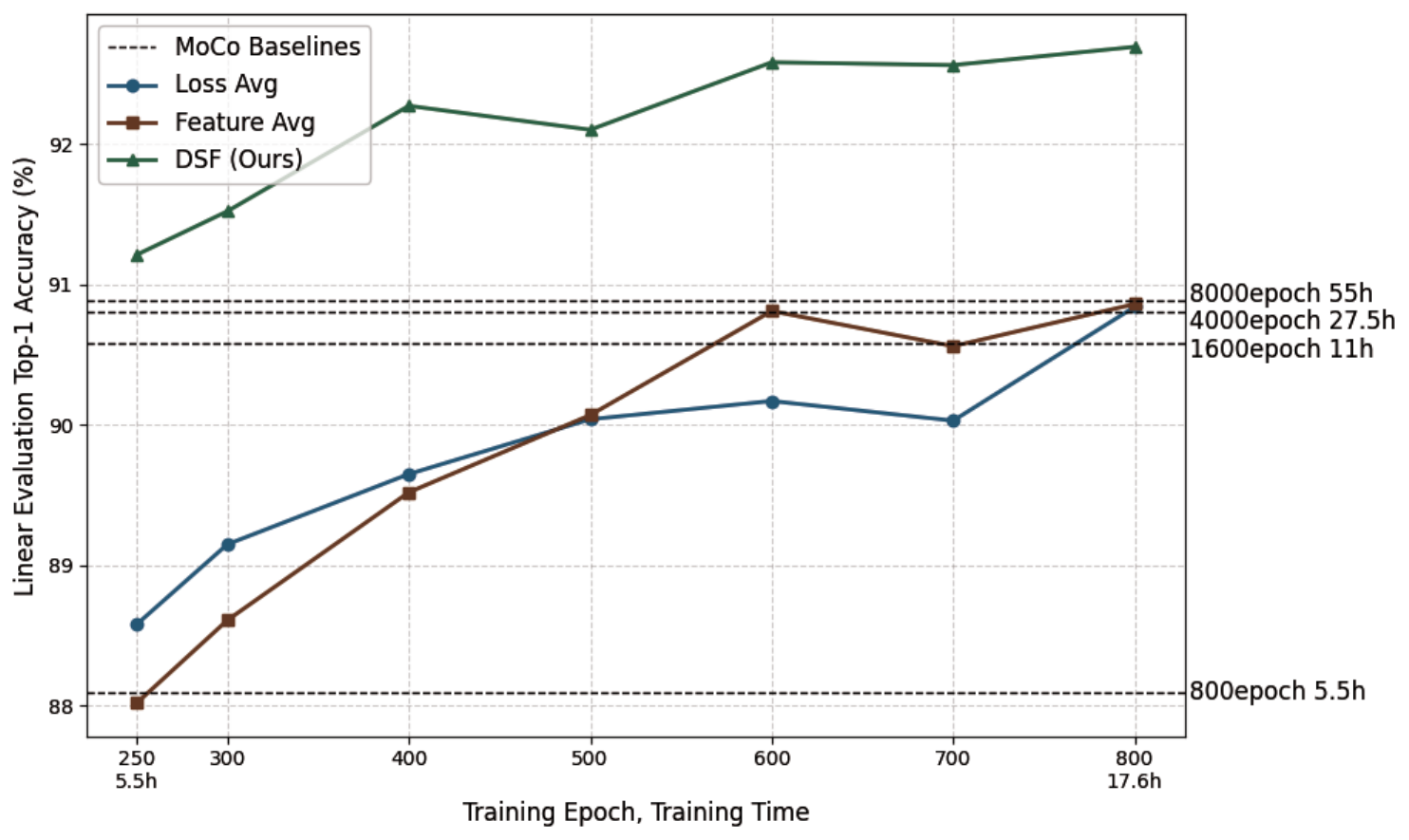}
        \caption{Linear evaluation accuracy with respect to training time.}
        \label{fig:linear_accuracy}
    \end{subfigure}
    \caption{\textbf{More Time Cost Comparison of performance over training time for kNN and linear evaluation.} 
        Both figures report accuracy for different training epochs and wall-clock time under the fixed GPU memory budget, with the number of views set to 8 ($M = 2m = 8$). 
        The right-hand side of each figure shows the elapsed time of MoCo baselines, while the $x$-axis indicates the time and epochs for the multi-view methods. 
        Note that the lowest dashed line (800 epochs, 5.5h) uses the same GPU memory and training time as the initial point (250 epochs, 5.5h) of each multi-view method. Compared to \textit{Loss\_avg} and \textit{Feature\_avg} at 800 epochs, DSF achieves about 2$\times$ faster convergence in the kNN setting and 3$\times$ faster convergence in the linear evaluation setting.}
    \label{fig:training_time_comparison}
\end{figure*}

\subsection{Original Contrastive Learning}

CL appears in many self-supervised learning methods. SimCLR~\cite{chen2020simple} leverages the InfoNCE loss~\cite{oord2018representation} to pull together representations of two augmented views while pushing apart all other samples. MoCo~\cite{he2020momentum} introduces a momentum encoder and a dynamic dictionary (negative queue) to maintain a large and consistent set of negative samples. DCL~\cite{yeh2022decoupled} decouples the positive and negative terms in the InfoNCE loss, removing negative-positive coupling to enable stable and high performance training with smaller batches and without a momentum encoder. MoCo v3~\cite{chen2021empirical} explores CL using a transformer-based model, specifically the Vision Transformer (ViT)~\cite{dosovitskiy2020image}.

%----------------------------------------------------
\subsection{Contrastive Learning with $M\ge3$ views} 

Existing $M \ge 3$ multi‐view CL methods have broadly fallen into two categories. The first category shapes the loss by averaging pairwise terms across views. \cite{shidani2024poly} generates $M$ views per a sample and aggregates their pairwise losses either arithmetically or geometrically. The second category is a method that deals with view representations at the feature level before computing the final objective. FastSiam~\cite{Fastsiam} averages the other ($M-1$) view features to construct a target of each view, and M3G~\cite{piran2024contrasting} adopts a multi‐marginal optimal‐transport loss by measuring the gap between the true $M$-tuple transport cost and its optimal polymatching cost. In contrast to these approaches, DSF integrates multi-views into a single collective loss formulation without incurring additional computational overhead.

%%%%%%%%%%%%%%%%%%%%%%%%%%%%%%%%%%%%%%%%%%%%%%%%%%%%%%%%%
\section{Preliminaries}
\label{section:prelim}

%-------------------------------------------------------------------------
\subsection{Contrastive Learning}
CL has emerged as a powerful paradigm in unsupervised and self-supervised learning, aiming to learn representations by distinguishing positive and negative pairs. Positive pair refers to semantically similar pair whereas negative pair to dissimilar ones. The core idea is to map positive samples closer while pushing negative samples apart in the embedding space. Given a positive pair and a set of negative samples, InfoNCE loss encourages the model to assign high similarity to the positive pair relative to the negatives. The loss can be formulated as follows:

\begin{equation}
    L^{\mathrm{cos}}_{i,\text{InfoNCE}} = -\log \frac
        { e^{(\mathrm{sim}(z_i,z^+_{i}) / \tau)} }
        { e^{(\mathrm{sim}(z_i,z^+_{i}) / \tau)}+\sum\limits_{j\neq i}^{K}e^{(\mathrm{sim}(z_i,z^{-}_{j}) / \tau)} } ,
  \label{eq:cl_loss}
\end{equation}
where $z^+_i$ is the embedding of the positive pair respect to $z_i$ and $z^-_j$s are embeddings of $K-1$ negative pairs. $\mathrm{sim}(\cdot, \cdot)$ is a similarity measure (e.g., dot product) and $\tau$ is a temperature parameter.

%---------------------------------------------------------------
\subsection{von Mises-Fisher Distribution}
vMF distribution is defined on the \((p\!-\!1)\)-dimensional unit hypersphere $\mathbb{S}_{p-1}$ and models the probability density of unit vectors with a specified mean direction $\mu$ and concentration parameter $\kappa$. Its probability density function for a unit vector $x \in \mathbb{S}_{p-1}$ is:

\begin{equation}
f_p(x;\mu,\kappa) = C_p(\kappa)\exp(\kappa \mu^T x),
\end{equation}
where \(C_p(\kappa)\) is the normalization constant. Higher $\kappa$ values yield distributions more concentrated around $\mu$, while $\kappa = 0$ corresponds to a uniform distribution on the unit hypersphere.

%-----------------------------------------------------
\subsection{Notation}
\label{subsection:notation}

We denote the batch size, number of views, and feature dimension by \( B \), \( M \), and \( p \), respectively. Each encoder is given \( m \) views, where \( m \) indicates the number of augmented views per encoder, resulting in a total of \( M = 2m \) views.

Let \( x \) be an image data, and \( x^l \) is denoted by its \( l \)-th augmented view. Given an encoder \( f_\theta(\cdot) \), the corresponding feature on the unit hypersphere is computed as \( z^l = f_\theta(x^l) / \| f_\theta(x^l) \| \), $l=1,\cdots,m$. In CL, there are $2 \cdot B$ features $z_i$, $i=1,\cdots,2B$. $z_i$ and $z_{B+i}$ are positively correlated. In multi-view CL, there are $2m \cdot B$ features $z_i^l$, $i=1,\cdots, B$ and $l=1,\cdots,m$. $\{z_i^l\}_{l=1}^m$ and $\{z_{B+i}^l\}_{l=1}^m$ are positively correlated.

The Kullback–Leibler (KL) divergence is written as \( D_{\mathrm{KL}}(\cdot \| \cdot) \), and we denote a vMF distribution by the calligraphic symbol \( \mathcal{D} \).

%%%%%%%%%%%%%%%%%%%%%%%%%%%%%%%%%%%%%%%%%%%%%%%%%%%%%%%%%
\section{Divergence-based similarity function}
\label{section:divergence-based_similarity_function}

In this section, we describe DSF in detail and demonstrate its relationship to cosine similarity, as well as its inherent temperature-free property.

% In Sec\ref{subsection:notation}, we formalize notations throughtout this paper.
% In Sec\ref{subsection:estimate_vmf_distribution}, we describe how to estimate vMF distribution parameters $\mu$ and $\kappa$.
% In Sec\ref{subsection:stabilization_of_kappa}, we propose the adjustment to kappa value for computational stabilization.
% In Sec\ref{subsection:divergence_similairty_function}, we describe the calculation of DSF.
% In Sec\ref{subsection:analysis}, we show that KL divergence-based similarity function between vMF distributions becomes equivalent to the cosine similarity function in InfoNCE Loss under certain conditions. Furthermore, we show why DSF does not need a temperature hyperparameter and why the cosine similarity function need temperature hyperparamter.

% --------------------------------------------------------
\subsection{Estimation of vMF Distribution}
\label{subsection:estimate_vmf_distribution}

First, we describe how to construct vMF distributions from multiple augmented views of an image.

Given data \( x \), we sample \( M \) augmented views and divide them into two groups of \( m \) samples each. These are passed through the encoder to obtain unit-normalized feature vectors. Each group is then used to estimate a vMF distribution, yielding two vMFs per a data.

Let \( z^1, \dots, z^m \) be the feature vectors from one group. The mean direction \( \mu \) and concentration \( \kappa \) are estimated as:
\begin{equation}
    \label{eq:estimation}
    \mu = \frac{\bar{z}}{\bar{R}}, \quad \kappa = A_p^{-1}(\bar{R}),
\end{equation}
where
\begin{equation}
    \bar{z} = \frac{1}{m} \sum\limits_{l=1}^{m} z^l, \quad \bar{R} = \|\bar{z}\|.
\end{equation}
As \( A_p^{-1} \) is intractable, we use the approximation from \cite{banerjee2005clustering}:
\begin{equation}
    \label{eq:kappa_approximation}
    \kappa = \frac{\bar{R}(p-\bar{R}^2)}{1-\bar{R}^2}.
\end{equation}

Let $g_{\text{est}}$ be the estimation function mapping features to vMF parameters:
\[
g_{\text{est}}(z^1, \dots, z^m) = (\mu, \kappa),
\]
and define the resulting vMF distribution as:
\begin{equation}
    \mathcal{D}\left(g_{\text{est}}(z^1, \dots, z^m)\right) = \mathcal{D}(\mu, \kappa).
\end{equation}
For data $i$, we denote:
\begin{equation}
    \mathcal{D}\left(g_{\text{est}}(z_i^1, \dots, z_i^m)\right) =  \mathcal{D}(\mu_i, \kappa_i) = \mathcal{D}_i.
\end{equation}

% -------------------------------------------------------------
\subsection{Stabilization of the Concentration Parameter $\kappa$}
\label{subsection:stabilization_of_kappa}

In this section, we describe a stabilization strategy for the concentration parameter \( \kappa \), which is critical for preventing numerical issues during training.

As training progresses, \( \kappa \) may grow excessively due to:

\begin{itemize}
    \item \textbf{High feature dimensionality:} Larger feature dimensions \( p \) result in sharper increase in \( \kappa \), especially with high-dimensional embeddings (e.g., \( D = 128 \)). The problem is further exacerbated when larger models are used.
    \item \textbf{Feature concentration:} As representations improve, features become more aligned and \( \bar{R} \) approaches 1, further sharp increasing \( \kappa \).
\end{itemize}

Though theoretically valid, such growth causes numerical instability in divergence computation. To mitigate this, we adopt two strategies:

\begin{enumerate}
    \item \textbf{Dimension-based normalization:} Normalize \( \kappa \) by the feature dimension \( p \), since \( \kappa \) grows roughly linearly with respect to \( p \) when \( \bar{R} \) is fixed (see Eq.~\ref{eq:kappa_approximation}).
    \item \textbf{Scaling mean resultant length:} Prevent \( \bar{R} \) from approaching 1 by applying a scaling factor \( \lambda_{\bar{R}} = 0.95 \):
    \begin{equation}
        \bar{R} \longrightarrow \lambda_{\bar{R}} \cdot \bar{R}
    \end{equation}
\end{enumerate}

These techniques are applied after estimating the vMF parameters, and they help stabilize training and prevent divergence-related issues in high-dimensional settings. Further details are provided in Appendix~\ref{appendix:stabilize_kappa}, Figure~\ref{fig:stabilize_kappa}.

\begin{table*}[t]
\centering
\begin{tabular}{lllcccc}
\toprule
\textbf{Dataset}      & \textbf{Base Model} & \textbf{Method}       & \textbf{Epoch} & \textbf{kNN (↑)} & \textbf{Linear Evaluation} (↑) & \textbf{mCE* (↓)} \\ \midrule
ImageNet-100 & MoCo v3    & -            & 300                       & 74.08                    & 83.42             &  36.67            \\
             & MoCo v3    & Loss\_avg    & 100                       & 76.62                    & 84.06             &  38.27             \\
             & MoCo v3    & Feature\_avg & 100                       & 76.94                    & 83.82             &  39.42             \\
             & MoCo v3    & DSF          & 100                       & \textbf{79.00}           & \textbf{85.80}    & \textbf{33.15}    \\ \midrule
CIFAR-10     & MoCo       & -            & 800                       & 88.44                    & 88.10             &  19.03       \\
             & MoCo       & Loss\_avg    & 250                       & 88.28                    & 88.58             &  19.90       \\
             & MoCo       & Featrue\_avg & 250                       & 87.47                    & 88.02             &  19.80       \\
             & MoCo       & DSF          & 250                       & \textbf{90.04}           & \textbf{91.21}    &  \textbf{16.20}   \\ \bottomrule
\end{tabular}
\caption{\textbf{Comparison of kNN, Linear Evaluation and mCE Results Under Equivalent Resource Settings.} All experiments are conducted under the same GPU memory and time constraints. The number of views is fixed at \( M = 2m = 8 \). The mean Corruption Error(mCE) was evaluated on CIFAR-10-C and ImageNet-100-C, with all models trained solely on clean(uncorrupted) data. }
\label{tab:knn-linear-corruption-results}
\end{table*}

% ----------------------------------------------------
\subsection{Divergence-based similarity function}
\label{subsection:divergence_similairty_function}

\subsubsection{Similarity function}
Similarity function $\text{sim}(\cdot,\cdot)$ is given by $\text{sim}(\cdot,\cdot) : \mathcal{X} \times \mathcal{X} \rightarrow \mathbb{R}$. For example, the cosine similarity function is given as $\text{sim}_{\text{cos}}(\cdot,\cdot) : \mathbb{S}_{p-1} \times \mathbb{S}_{p-1} \rightarrow \mathbb{R}$, $\text{sim}_{\text{cos}}(x,y) = x^{T}y$. Accordingly, we define the DSF similarity function to handle multi-view features as $\text{sim}_{\text{div}}(\cdot,\cdot) : \mathbb{S}_{p-1}^m \times \mathbb{S}_{p-1}^{m} \rightarrow \mathbb{R}$. More precisely,

\begin{equation}
\begin{split}
    &\text{sim}_{\text{div}}(\mathbb{X}, \mathbb{Y}) = -D_{\mathrm{KL}} \Bigl( \mathcal{D}\bigl(g_{\text{est}}(\mathbb{X})\bigr)\bigm\Vert \mathcal{D}\bigl(g_{\text{est}}(\mathbb{Y})\bigr) \Bigr) .\\
\end{split}
\end{equation}

\subsubsection{KL Divergence of vMF Distribution}
Let $\mathcal{D}_i = \mathcal{D}(\mu_i,\kappa_i)$ and $\mathcal{D}_j = \mathcal{D}(\mu_j,\kappa_j)$ be any two vMF distributions. The KL divergence value between $\mathcal{D}_i$ and $\mathcal{D}_j$ is expressed as the sum of three terms:

\begin{equation}
\begin{split}
    \label{eq:kld}
    &D_{\mathrm{KL}}(\mathcal{D}_i || \mathcal{D}_j) \\
    &= \log \frac{f_p(A_p(\kappa_i)\mu_i; \mu_i, \kappa_i)}{f_p(A_p(\kappa_i)\mu_i; \mu_j, \kappa_j)} \\
    &= \left( \frac{p}{2} - 1 \right) \log \frac{\kappa_i}{\kappa_j} 
    + \log \frac{I_{\frac{p}{2}-1}(\kappa_j)}{I_{\frac{p}{2}-1}(\kappa_i)} \\
    & \qquad \qquad \qquad + A_p(\kappa_i) \cdot \left( \kappa_i - \kappa_j \cdot \mu_i^T \mu_j \right) .\\
    \\
\end{split}
\end{equation}
% Detailed explanation and evaluation of $A_p(\cdot)$ and $I_p(\cdot)$ are provided in Appendix~\ref{appendix:further_explanation}. The KL divergence ranges from 0 to infinity, it is not a distance function between two distributions. As the KL divergence value closes to zero, the distributions become similar, which means $\mu_i \approx \mu_j$ and $\kappa_i \approx \kappa_j$.

Figure~\ref{fig:three_terms} in the Appendix illustrates the behavior of each term as well as the overall KL divergence values. Detailed explanation and evaluation of $A_p(\cdot)$ and $I_p(\cdot)$ are provided in Appendix~\ref{appendix:further_explanation}.

The KL divergence ranges from 0 to infinity. As the KL divergence approaches zero, the two distributions become similar, indicating that $\mu_i \approx \mu_j$ and $\kappa_i \approx \kappa_j$. Conversely, a large KL divergence implies that the two distributions are highly dissimilar. Therefore, we define the DSF as the negative KL divergence. We use following expression for notational convenience throughout the remainder of the paper:

\begin{equation}
    \text{sim}_{\text{div}}(\mathcal{D}_i, \mathcal{D}_j) = - D_{\mathrm{KL}}(\mathcal{D}_i || \mathcal{D}_j)
\end{equation}
Now, we define the loss function with DSF as:
\begin{equation}
L_{i,\text{InfoNCE}}^{\text{div}} = 
-\log\left(
\frac{
    e^{\text{sim}_{\text{div}}(\mathcal{D}_i,\mathcal{D}_i^+)}
}{
    e^{\text{sim}_{\text{div}}(\mathcal{D}_i,\mathcal{D}_i^+)} 
    + \sum\limits_{j \neq i}^{K} e^{\text{sim}_{\text{div}}(\mathcal{D}_i,\mathcal{D}_j)}
}
\right)
\label{eq:dsf_loss}
\end{equation}

Here, \( \mathcal{D}_i^+ \) denotes the positive distribution paired with \( \mathcal{D}_i \). The negative distributions are denoted by \( \mathcal{D}_j \), where \( K \) represents the number of negative distributions. Note that the \( K \) may vary depending on the training framework.

%------------------------------------------------------
\subsection{Analysis : Relation with Cosine Similarity Function}
\label{subsection:analysis}

Next, we establish the connection between the conventional cosine similarity function and DSF.

\begin{theorem}[Relation between Divergence-based and Cosine Similarities]
\label{theorem:relation_of_similarity_function}
    If $m=1$ and all concentration parameters $\kappa$ have the same value $\kappa^*$  and $A_p(\kappa^*)\cdot\kappa^* = 1/\tau$ then,
        \begin{equation}
            \mathrm{sim}_{\mathrm{div}}(\mathcal{D}(z_i), \mathcal{D}(z_j)) =  \mathrm{sim}_{\mathrm{cos}}(z_i, z_j) / \tau + constant
        \end{equation}
    
    Furthermore, InfoNCE Loss with $\mathrm{sim}_{\mathrm{div}}$ becomes equivalent to that using $\mathrm{sim}_{\mathrm{cos}}$ one:
    \begin{equation}
        L^{\mathrm{div}}_{\mathrm{InfoNCE}} = L^{\mathrm{cos}}_{\mathrm{InfoNCE}}
    \end{equation}
\end{theorem}

\begin{proof}
    See Appendix~\ref{appendix:proof_relation}.
\end{proof}

Theorem \ref{theorem:relation_of_similarity_function} states that when we have two views and $\kappa$ is set under certain conditions, DSF becomes cosine similarity with some constant value. Furthermore, this leads to the fact that InfoNCE losses become equivalent. 

We next provide an explanation of why cosine similarity requires a temperature scaling term for effective optimization, while DSF inherently eliminates the need for such adjustment.

\begin{proposition}
    Let \( K \) be the number of negative pairs, and assume all negative similarities have the same value.
    Let \( s^+ \) and \( s^- \) denote the positive and negative similarity, respectively.
    Then, the InfoNCE loss becomes:
    
    \begin{equation}
    \begin{split}
        L_{\mathrm{InfoNCE}}
        &= -\log \frac{\exp(s^+ - s^-)}{\exp(s^+ - s^-) + K} \\
        (s^+ - s^-) &\in 
         \left\{
         \begin{matrix} 
         [-\frac{2}{\tau},\frac{2}{\tau}]  & \text{if}  \; \mathrm{sim} = \mathrm{sim}_{\mathrm{cos}} / \tau \\
         (-\infty,\infty) & \text{if} \; \mathrm{sim} = \mathrm{sim}_{\mathrm{div}} 
         \end{matrix}
         \right.
    \end{split}
    \end{equation}

    \label{proposition}

    In the optimal case for cosine similarity, we have \( s^+ = +1 \) and \( s^- = -1 \), resulting in:
    
    \begin{equation}
    \begin{split}
        L_{\text{InfoNCE}}
        &= -\log \frac{\exp(2/\tau)}{\exp(2/\tau) + K}\\
        &= 6.3196, \quad \text{if} \; \tau=1.0, \, K=4096 \\
    \end{split}
    \end{equation}
\end{proposition}

We show (see Proposition~\ref{proposition}) that cosine similarity alone fails to produce appropriate loss values without a temperature term \( \tau \), due to the limited similarity margin \( (s^+ - s^-) \) between positive and negative pairs. The temperature parameter compensates for this by amplifying the margin. As shown in Table~\ref{table:infonce_loss}, InfoNCE loss with cosine similarity produces suboptimal values for all $K$ when \( \tau \) is omitted, highlighting the need for careful tuning. In contrast, DSF achieves appropriate loss values without requiring such tuning. See Figure~\ref{figure:s^+_s^-} for similarity in DSF.

\begin{table}
    \centering
    \begin{tabular}{c|ccc}
    \toprule
    & \multicolumn{3}{c}{$L_{\text{InfoNCE}}$} \\ \midrule
    $\tau$ & K=256  & K=4096  & K=65536  \\ \midrule
    1.0    & 3.573  & 6.319   & 9.090    \\
    0.5    & 1.738  & 4.331   & 7.091    \\
    0.2    & 0.011  & 0.170   & 1.380    \\
    0.1    & 0.000  & 0.000   & 0.0001   \\ \bottomrule
    \end{tabular}
    \caption{InfoNCE loss values under optimal similarity ($s^+=1$, $s^-=-1$) for varying temperature $\tau$ and the number of negative pairs $K$. The number of negatives \( K \) depends on the framework: SimCLR uses \( K = 2B - 2=254 \) when $B=128$, while MoCo sets \( K \) by the memory queue size (4096–65536).}
    \label{table:infonce_loss}
\end{table}

\begin{algorithm}[tb]
\caption{Three Multi-View Methods}
\label{algorithm:multi-view-methods}
\textbf{Input}: Data $x_i \in \mathbb{R}^{C \times H \times W}$, encoders $f_{\theta}$ and $f_{\theta'}$. \\
\textbf{Parameter}: Number of views $M = 2m$. \\
\textbf{Output}: Divergence-based similarity between positive pairs.
\begin{algorithmic}[1]
\setlength{\itemsep}{0.5em}  % 줄 간격 설정

\STATE Apply $M=2m$ augmentations to $x_i$, yielding $x_i^1, \dots, x_i^m$ and $x_{B+i}^1, \dots, x_{B+i}^m$
\STATE Extract features: \\
$\displaystyle z_i^l = \frac{f_{\theta}(x_i^l)}{\|f_{\theta}(x_i^l)\|},\,z_{B+i}^l = \frac{f_{\theta'}(x_{B+i}^l)}{\|f_{\theta'}(x_{B+i}^l)\|}\,\text{for } l = 1,\dots,m $

\STATE \quad \textbf{if} multi-view method = \texttt{Loss\_avg}: \\
\qquad $\displaystyle L_{i,\text{Loss\_avg}} = \frac{1}{m^2} \sum_{l=1}^{m} \sum_{l'=1}^{m} L_{i, \text{InfoNCE}}^{\cos}(z_i^l, z_{B+i}^{l'}) $

\STATE \quad \textbf{if} multi-view method = \texttt{Fea\_avg}: \\
\qquad $\displaystyle z_i^{\text{avg}} = \frac{1}{m} \sum_{l=1}^{m} z_i^l, \quad
z_{B+i}^{\text{avg}} = \frac{1}{m} \sum_{l=1}^{m} z_{B+i}^l $ \\
\qquad $\displaystyle L_{i,\text{Fea\_avg}} = L_{i,\text{InfoNCE}}^{\cos}(z_i^{\text{avg}}, z_{B+i}^{\text{avg}}) $

\STATE \quad \textbf{if} multi-view method = \texttt{DSF}: \\
\qquad Estimate vMF distribution parameters as in\\
\qquad Eq.~(3)--(4):

\qquad \quad $\displaystyle (\mu_i, \kappa_i) \leftarrow g_{\text{est}}\left(\{z_i^1, \dots, z_i^m\}\right)$ \\
\qquad \quad $\displaystyle (\mu_{B+i}, \kappa_{B+i}) \leftarrow g_{\text{est}}\left(\{z_{B+i}^1, \dots, z_{B+i}^m\}\right)$

\STATE \qquad Compute divergence-based similarity: \\
\qquad \quad $\displaystyle \mathrm{sim}_{\mathrm{div}}(\mathcal{D}_i, \mathcal{D}_{B+i}) = -D_{\mathrm{KL}}(\mathcal{D}_i \, \| \, \mathcal{D}_{B+i}) $ \\
\qquad \quad $\displaystyle L_{i,\text{DSF}} = L_{i,\text{InfoNCE}}^{\mathrm{div}}\left(\mathcal{D}(\mu_i, \kappa_i), \mathcal{D}(\mu_{B+i}, \kappa_{B+i})\right) $

\end{algorithmic}
\end{algorithm}

%%%%%%%%%%%%%%%%%%%%%%%%%%%%%%%%%%%%%%%%%%%%%%%%%%%%%%%%%
\section{Experiments}
\label{section:experiments}

% ------------------------------------------------------------------
\subsection{Experimental Setup}

\paragraph{Datasets.}
We conduct experiments on two standard benchmarks: CIFAR-10~\cite{krizhevsky2009learning} and ImageNet-100, a 100-class subset of ImageNet~\cite{deng2009imagenet}. These datasets allow us to evaluate the proposed method on both small- and large-scale visual recognition tasks.

We investigated transfer learning performance on the Food-101 dataset~\cite{bossard2014food}, CIFAR-10
and CIFAR-100~\cite{krizhevsky2009learning}, FGVC Aircraft~\cite{maji2013fine}, the Describable Textures Dataset (DTD)~\cite{cimpoi2014describing} and Oxford-IIIT Pets~\cite{parkhi2012cats}. 

We adopt CIFAR-10-C~\cite{hendrycks2019benchmarkingneuralnetworkrobustness} and ImageNet-100-C as corruption benchmarks to evaluate the robustness of the learned representations. The ImageNet-100-C dataset is constructed by selecting the 100 classes used in ImageNet-100 from the original ImageNet-C benchmark. For both datasets, we use only the corrupted samples with severity level 1 for evaluation.

\paragraph{Model Architectures.}
We follow the MoCo v2~\cite{chen2020improved} framework, using ResNet-18 as the encoder for CIFAR-10, with a 128-dimensional projection head. For ImageNet-100, we adopt MoCo v3~\cite{chen2021empirical} with a ViT-S/16 backbone to evaluate scalability across both CNN and transformer architectures.

\paragraph{Data Augmentation.}
For data augmentation, we apply a combination of RandAugment~\cite{cubuk2020randaugment}, ColorJitter, Random Grayscale, Gaussian Blur, and Horizontal Flip, following common practice in CL setups.
% ------------------------------------------------------------------
\subsection{Evaluations}
We evaluate the model using two standard protocols: k-Nearest Neighbors (kNN) and Linear Evaluation. The kNN evaluation measures similarity preservation in the embedding space, while linear evaluation assesses the representational quality via a trained classifier.

\paragraph{kNN Evaluation.}
Using $k=200$, we conduct kNN evaluation on CIFAR-10 and ImageNet-100 with features from the trained encoder. DSF method achieves 79.00\% on ImageNet-100 and 90.04\% on CIFAR-10, outperforming MoCo v3 baselines (74.08\% and 88.44\%, respectively) and other multi-view methods. See Table~\ref{tab:knn-linear-corruption-results}.

\paragraph{Linear Evaluation.}
We freeze the pretrained encoder and train a linear classifier on top. DSF achieves 85.80\% on ImageNet-100 and 91.21\% on CIFAR-10, surpassing MoCo v3 results of 83.42\% and 88.10\% and also other multi-view methods. These results indicate that DSF produces more discriminative and transferable representations. See Table~\ref{tab:knn-linear-corruption-results}.

\paragraph{Transfer Learning.}
Transfer learning is one of the ultimate goals of SSL. We evaluate the transfer performance on five natural image datasets under the fine-tuning setting, where all model weights are updated without freezing. More details are in Appendix~\ref{appendix:transfer_learning}

\paragraph{Robustness under Distribution Shift.}
To assess the robustness of the learned representations under distribution shift, we evaluate the models on common corruption benchmarks: CIFAR-10-C and ImageNet-100-C, following the protocol of~\cite{hendrycks2019benchmarkingneuralnetworkrobustness}. We use the mean Corruption Error (mCE) as the evaluation metric, measured using the model parameters obtained from linear evaluation on CIFAR-10 and ImageNet-100, respectively. No additional fine-tuning is performed during this evaluation.

As shown in Table~\ref{tab:knn-linear-corruption-results}, DSF consistently outperforms both the baseline and other multi-view contrastive learning methods, achieving an mCE of 16.20\% on CIFAR-10-C and 33.15\% on ImageNet-100-C. These results indicate that the joint structure–capturing capability of DSF also enhances the model’s robustness to input corruptions.

% We further assess the robustness of learned representations on corrupted datasets. Without any additional fine-tuning, we conduct linear evaluation on ImageNet-100-C and CIFAR-10-C dataset following the protocol of~\cite{hendrycks2019benchmarkingneuralnetworkrobustness}. The mean Corruption Error (mCE) is adopted as the evaluation metric. As summarized in Table~\ref{table:OOD}, DSF consistently outperforms other multi-view contrastive learning methods by a significant margin on both CIFAR-10-C (16.20\%) and ImageNet100-C (33.15\%). These results suggest that DSF effectively captures class-level representations that are resilient to distributional shifts.

\begin{figure}[tb]
    \centering
    \begin{subfigure}[t]{1.0\linewidth}
        \centering
        \includegraphics[width=\linewidth]{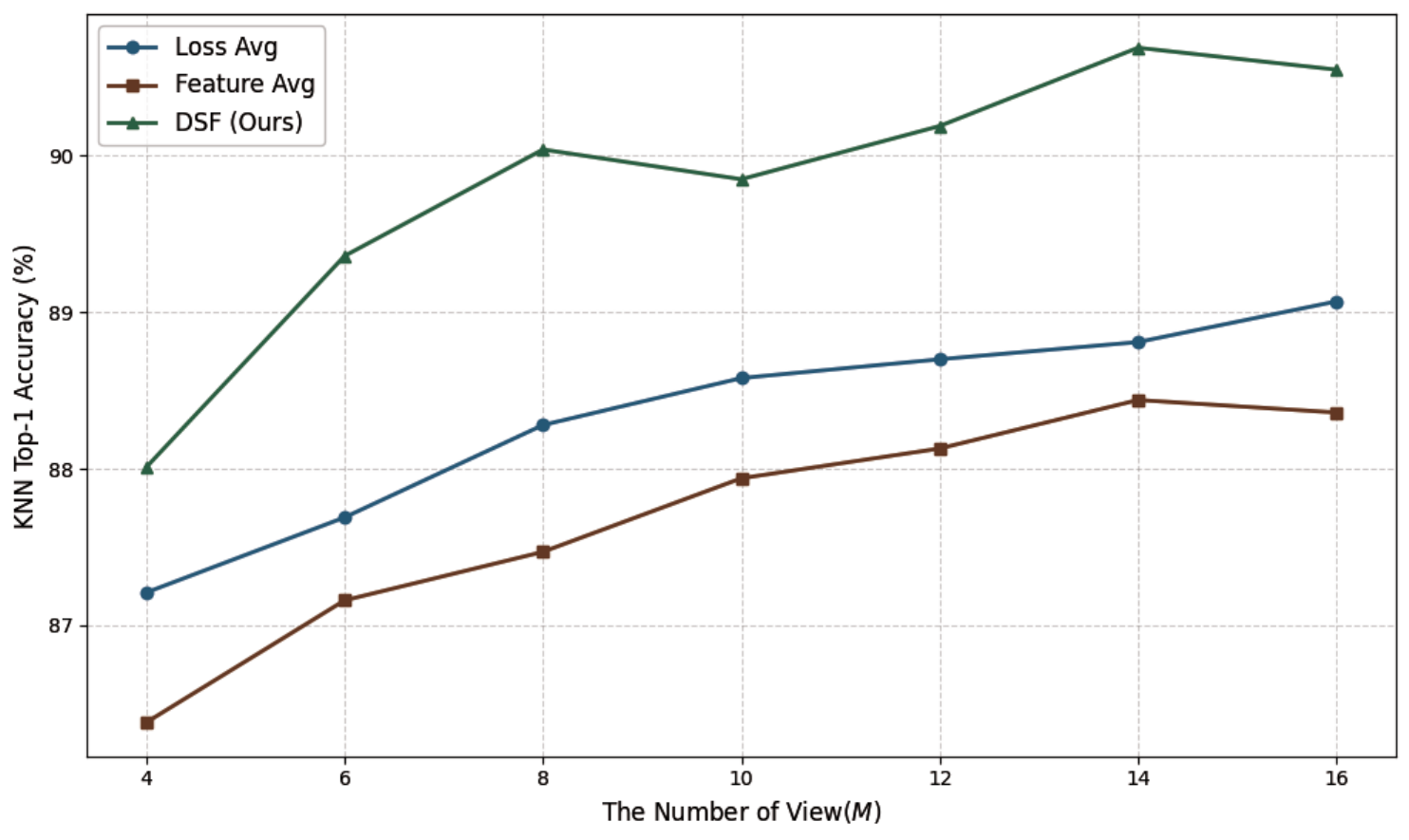}
        \caption{kNN Top-1 Accuracy with respect to the number of view.}
        \label{figure:number_of_view_kNN}
    \end{subfigure}
    
    \vspace{0.5em}

    \begin{subfigure}[t]{1.0\linewidth}
        \centering
        \includegraphics[width=\linewidth]{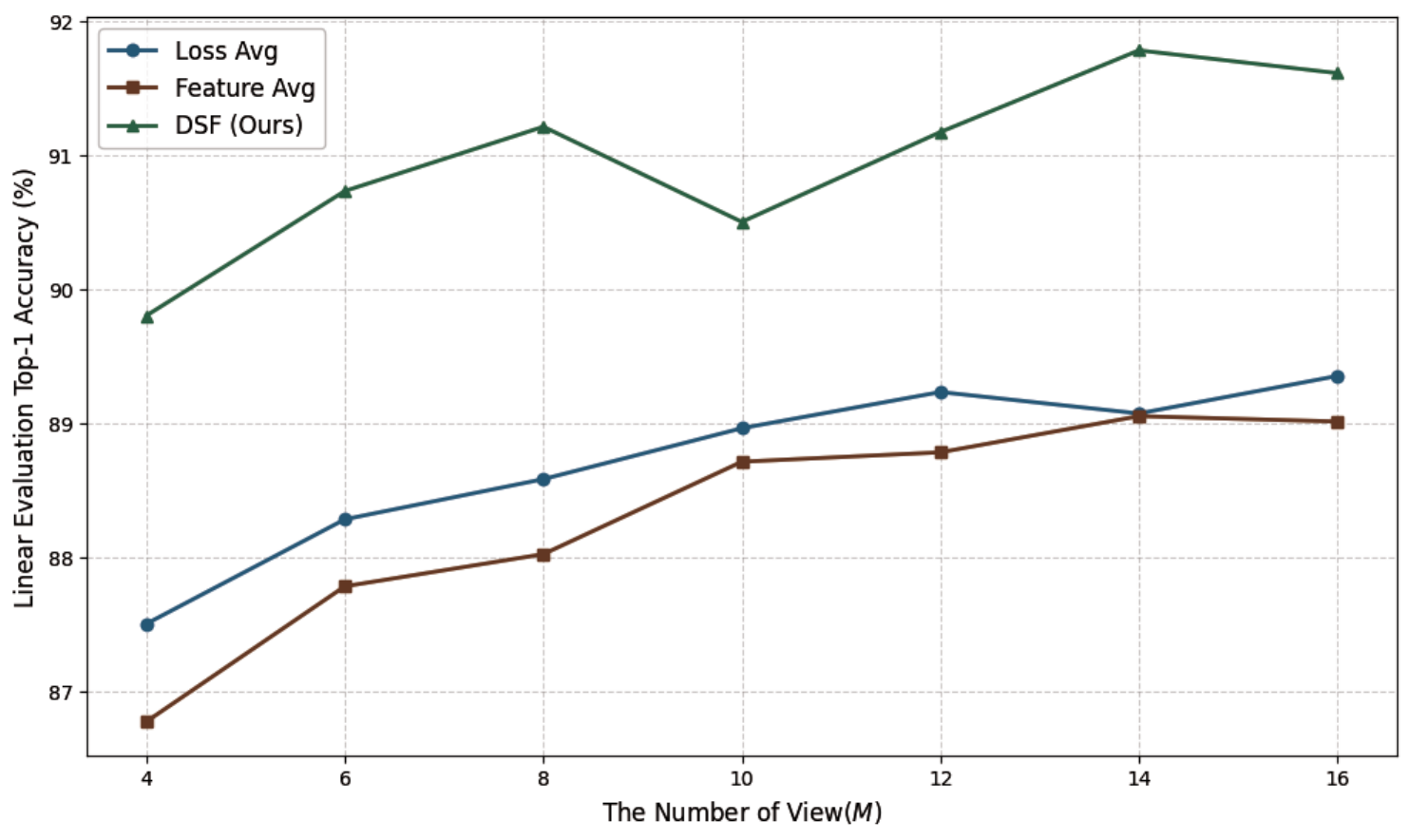}
        \caption{Linear Evaluation Top-1 Accuracy with respect to the number of view.}
        \label{figure:number_of_view_linear_evaluation}
    \end{subfigure}
    
    \caption{Comparison of kNN and linear evaluation accuracies according to the number of views in three multi-view learning methods.}

    \label{figure:number_of_view}
\end{figure}

\begin{table*}[t]
\centering
\begin{tabular}{llcccccccc}
\toprule
\multicolumn{1}{c}{\multirow{2}{*}{\textbf{Model}}} & \multicolumn{1}{c}{\multirow{2}{*}{\textbf{Method}}} & \multicolumn{6}{c}{\textbf{Dataset (Top-1 Accuracy \%)}}                                              \\ \cmidrule(l){3-8} 
\multicolumn{1}{c}{} & \multicolumn{1}{c}{} & CIFAR10              & CIFAR100           & Pet               & Aircraft          & DTD               & Average \\ \midrule
MoCo v3              & -                    & 96.36                & 80.80              & \underline{84.11} & \underline{47.58} & \underline{59.68} & \underline{73.70}   \\
MoCo v3              & Loss\_avg            & \underline{96.76}    & \underline{81.55}  & 83.11             & 47.17             & 58.67             & 73.45   \\
MoCo v3              & Feature\_avg         & 96.72                & 81.06              & 82.80             & 46.45             & 57.66             & 72.93   \\
MoCo v3              & DSF                  & \textbf{97.11}       & \textbf{81.76}     & \textbf{84.17}    & \textbf{49.49}    & \textbf{60.59}    & \textbf{74.62}   \\ \bottomrule
\end{tabular}
\caption{\textbf{Comparison of transfer learning performance.} Top-1 accuracy (\%) on five downstream datasets using MoCo v3 representations pre-trained on ImageNet-100. Bold indicates the best performance, and underline indicates the second best.}
\label{table:transfer_learning}
\end{table*}

%-----------------------------------------------------------
\subsection{Same Resource Comparison}

DSF samples multiple views per image, which may increase computational cost. To ensure fair comparison, we fix the product of batch size and number of views (\( B \times M \)) across all methods, maintaining equal GPU memory usage. For example, in ImageNet-100 experiments, MoCo v3 uses \(512 \times 2\), while DSF and other multi-view baselines use \(128 \times 8\), preserving memory usage. In CIFAR-10 experiments, MoCo uses \(1024 \times 2\), while DSF and other multi-view baselines adopt \(256 \times 8\).

% This strategy is also applied to CIFAR-10. 

To align training time, we reduce epochs for multi-view methods proportionally. This keeps runtime differences within 1 hour for ImageNet-100 (under 7\%) and 30 minutes for CIFAR-10 (under 10\%), ensuring that performance gains arise from DSF itself rather than extra computation.

%-------------------------------------------------------------------------
\subsection{More Resource Comparison}

\subsubsection{More Time Cost: Faster Convergence with Divergence Similarity}

Self-supervised learning often requires longer training to converge, due to the difficulty of learning meaningful representations without labels~\cite{chen2020simple, wang2021solving, wang2021dense}. To investigate convergence behavior, we compare model accuracies over elapsed training time (Figure~\ref{fig:training_time_comparison}). While MoCo requires 55 hours for 8000 epochs, DSF achieves similar accuracy in just 11 hours (400 epochs), using only 20\% of the time. Compared to \textit{Loss\_avg} and \textit{Feature\_avg}, DSF also reaches similar accuracy with half the epochs and maintains a consistently higher performance throughout training. These results highlight the efficiency of DSF in optimizing self-supervised objectives.

\subsubsection{More GPU Memory: Consistent Superiority Across View Counts}

In Figure~\ref{figure:number_of_view}, we examine how the number of augmented views affects representation quality, as increasing views generally improves performance~\cite{tian2020contrastive, shidani2024poly, Fastsiam}. In DSF, more views lead to more accurate estimation of the vMF parameters \( \mu \) and \( \kappa \), due to reduced variance. As shown in Figure~\ref{figure:number_of_view}, DSF consistently outperforms other multi-view methods across different numbers of views, clearly demonstrating the superiority of DSF in multi-view representation learning.

%---------------------------------------------------------------------

% \usepackage{multirow}
% \begin{table}[]
% \begin{tabular}{ccc}
%                                & model       & mCE(\%) \downarrow \\ \hline
% \multirow{CIFAR-10-C}     & MoCo        & 19.03  \\
%                                & MoCo\_pwe    & 19.90  \\
%                                & MoCo\_avg    & 19.80  \\
%                                & DSF         & \textbf{16.20}  \\ \hline
% \multirow{4}{*}{ImageNet-100-C} & MoCo\_v3     & 36.67  \\
%                                & MoCo\_v3\_pwe & 38.27  \\
%                                & MoCo\_v3\_avg & 39.42  \\
%                                & DSF         & \textbf{33.15}  \\ \hline
% \end{tabular}
% \label{table:OOD}
% \end{table}

%---------------------------------------------------------------------
\subsection{Temperature Free}
While prior works~\cite{zhang2021temperature, zhang2022dual, kukleva2023temperature} emphasize the importance of the temperature parameter \( \tau \) in CL, DSF does not require it (see Proposition~\ref{proposition}) and even shows degraded performance when \( \tau \) is introduced. As shown in Table~\ref{table:temperature}, the best results are consistently achieved when \( \tau = 1 \), indicating that DSF is inherently stable without temperature tuning.

\begin{table}
\centering
\begin{tabular}{lccccccccccc}
\toprule
$\tau$       & 0.1   & 0.9   & 1              & 1.1   & 5     \\ \midrule
kNN               & 27.71 & 89.83 & \textbf{90.04} & 89.80 & 38.44 \\
Lin. Eval. & 22.02 & 91.18 & \textbf{91.21} & 91.09 & 56.33 \\ \bottomrule
\end{tabular}
\caption{\textbf{Temperature Experiment}. Training MoCo with DSF under varying temperatures on CIFAR-10.}
\label{table:temperature}
\end{table}

%%%%%%%%%%%%%%%%%%%%%%%%%%%%%%%%%%%%%%%%%%%%%%%%%%%%%%%%%%%%%%

\begin{figure}[t]
    \centering
    \includegraphics[width=1.0\linewidth]{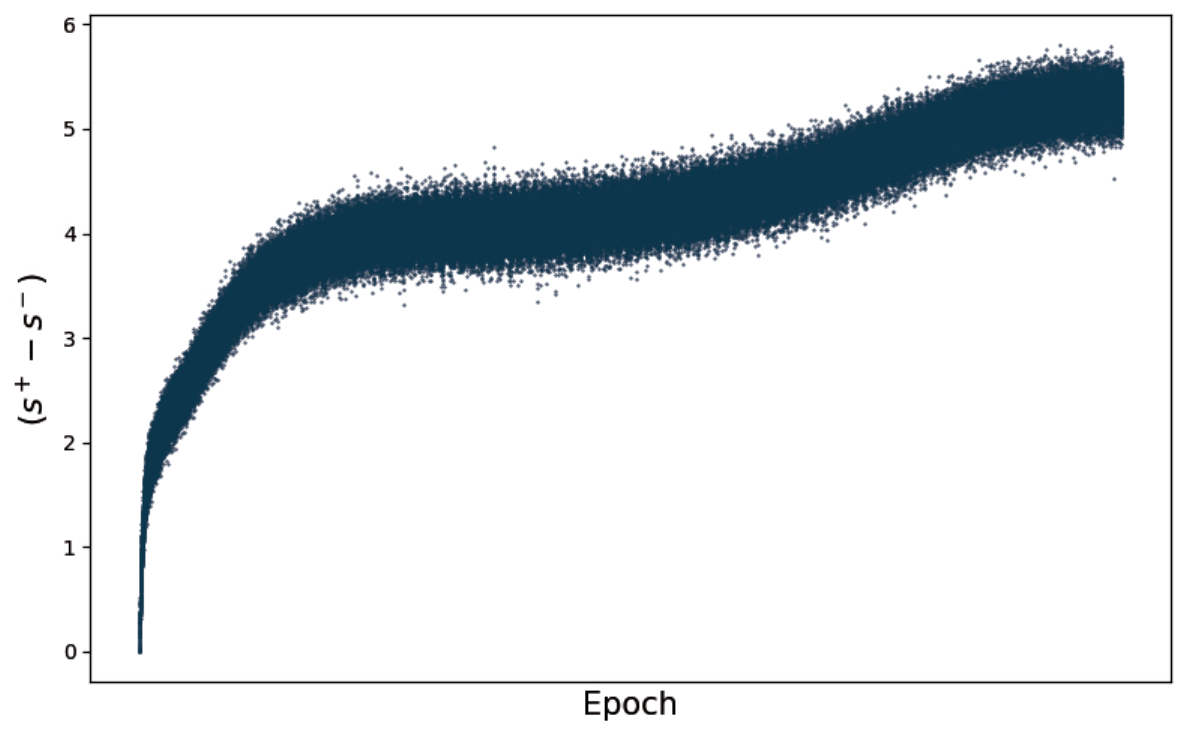}
    \caption{Training dynamics of the similarity margin $(s^+ - s^-)$. Values are averaged over mini-batches throughout training. This plot shows the similarity margin of MoCo v3 with DSF ($128 \times 8$) configuration on ImageNet-100.}
    \label{figure:s^+_s^-}
\end{figure}

%%%%%%%%%%%%%%%%%%%%%%%%%%%%%%%%%%%%%%%%%%%%%%%%%%%%%%%%%
\section{Conclusion}
\label{section:conclusion}
In this work, we propose a novel similarity function, named the Divergence-based Similarity Function (DSF), which captures joint relationships among multiple features. Our method demonstrates strong performance under both equal and increased resource settings, surpassing baseline and other multi-view methods. In addition, it produces highly transferable and robust representations, as evidenced by its effectiveness in transfer learning and resilience to data corruptions. We also establish its connection to cosine similarity and demonstrate that DSF performs best with a fixed temperature of 1, eliminating the need for tuning.

DSF is defined using the von Mises–Fisher distribution and Kullback–Leibler divergence. While effective, other $f$-divergences such as Rényi, $\chi^\alpha$, and Jensen–Shannon divergence remain unexplored, offering promising directions for future research.

Beyond these avenues, a natural extension is a deeper theoretical study of DSF. In particular, examining its generalization behavior and convergence in high-dimensional regimes could provide stronger foundations for its empirical success. Such analysis may also clarify why DSF exhibits superior robustness and transferability compared to cosine similarity.

%%%%%%%%%%%%%%%%%%%%%%%%%%%%%%%%%%%%%%%%%%%%%%%%%%%%%%%%%%%%%%%%%%%%%%
\section*{Acknowledgements}
This work is supported by the National Research Foundation of Korea(NRF) grant funded by the Korea government(MSIT) (RS-2024-00421203, RS-2024-00406127, NRF-2016K2A9A2A13003815) and 
IITP(2021-0-00077). We would also like to thank Imseong Park, Hyunsung Bae, Jongyeon Lee and Jaesuk Kim,  for their valuable feedback on an earlier version of this manuscript.

\bibliography{aaai2026}
\clearpage
%%%%%%%%%%%%%%%%%%%%%%%%%%%%%%%%%%%%%%%%%%%%%%%%%%%%%%%%%
\appendix

\section{Appendix: Proof of Theorem~\ref{theorem:relation_of_similarity_function}}
\label{appendix:proof_relation}

We prove that the divergence-based similarity function becomes equivalent to the cosine similarity function in InfoNCE loss under specific conditions.

Assume $m=1$, $\kappa_1 = \kappa_2 = \kappa$, and $A_p(\kappa) \cdot \kappa = 1/\tau$. Then, the divergence-based similarity simplifies as:
\begin{equation}
    \text{sim}_{\text{div}}(\mathcal{D}(z_i), \mathcal{D}(z_j)) = \frac{1}{\tau} \text{sim}_{\text{cos}}(z_i, z_j) + C,
\end{equation}
where $C$ is a constant independent of $(z_i, z_j)$ pair.

Substituting this into the InfoNCE loss:
\begin{equation}
\begin{split}
    &L_{i,\text{InfoNCE}}^{\text{div}} \\
    &= -\log\left(
    \frac{
        e^{\text{sim}_{\text{div}}(\mathcal{D}_i,\mathcal{D}_i^+)}
    }{
        e^{\text{sim}_{\text{div}}(\mathcal{D}_i,\mathcal{D}_i^+)} 
        + \sum\limits_{j \neq i}^{K} e^{\text{sim}_{\text{div}}(\mathcal{D}_i,\mathcal{D}_j)}
    }
    \right) \\
    &= -\log\left(
    \frac{
        e^{\frac{1}{\tau}\text{sim}_{\text{cos}}(z_i,z_i^+) + C}
    }{
        e^{\frac{1}{\tau}\text{sim}_{\text{cos}}(z_i,z_i^+) + C} 
        + \sum\limits_{j \neq i}^{K} e^{\frac{1}{\tau}\text{sim}_{\text{cos}}(z_i,z_j) + C}
    }
    \right) \\
    &= -\log\left(
    \frac{
        e^{\frac{1}{\tau}\text{sim}_{\text{cos}}(z_i,z_i^+)}
    }{
        e^{\frac{1}{\tau}\text{sim}_{\text{cos}}(z_i,z_i^+)} 
        + \sum\limits_{j \neq i}^{K} e^{\frac{1}{\tau}\text{sim}_{\text{cos}}(z_i,z_j)}
    }
    \right) \\
    &= L_{i,\text{InfoNCE}}^{\text{cos}}.
\end{split}
\end{equation}

\qed

%%%%%%%%%%%%%%%%%%%%%%%%%%%%%%%%%%%%%%%%%%%%%%%%%%%%%%%%%
\section{Feature Averaging Also Encodes Only Pairwise Relationships}
\textit{Loss\_avg} is clearly a pairwise-based method. We further demonstrate that \textit{Feature\_avg}, despite aggregating multiple views, still considers only pairwise relationships among features, as can be shown through a straightforward calculation.

\begin{align}
    z_i^{\text{avg}} &= \frac{1}{m} \sum_{l=1}^{m} z_i^l, \quad
    z_{B+i}^{\text{avg}} = \frac{1}{m} \sum_{l=1}^{m} z_{B+i}^l \notag
\end{align}

\begin{align}
    \text{sim}_{\text{cos}}(z_i^{\text{avg}}, z_{B+i}^{\text{avg}}) 
    &= (z_i^{\text{avg}})^T z_{B+i}^{\text{avg}} \notag \\
    &= \left( \frac{1}{m} \sum_{l=1}^{m} z_i^l \right)^T \left( \frac{1}{m} \sum_{l'=1}^{m} z_{B+i}^{l'} \right) \notag \\
    &= \frac{1}{m^2} \sum_{l=1}^{m} \sum_{l'=1}^{m} (z_i^l)^T z_{B+i}^{l'}
\end{align}

As shown above, the final similarity score is an unweighted average over all pairwise cosine similarities between individual view features \( z_i^l \) and \( z_{B+i}^{l'} \). Thus, \textit{Feature\_avg} does not model joint structure beyond pairwise relationships.

%%%%%%%%%%%%%%%%%%%%%%%%%%%%%%%%%%%%%%%%%%%%%%%%%%%%%%%%%

%%%%%%%%%%%%%%%%%%%%%%%%%%%%%%%%%%%%%%%%%%%%%%%%%%%%%%%%%
\section{Further Explanation of DSF}
\label{appendix:further_explanation}

\subsection{Definitions}
In here, we introduce the definition of $A_p(\cdot)$ and $I_p(\cdot)$ in Eq.~\eqref{eq:estimation} and Eq.~\eqref{eq:kld} respectively. Here, $\Gamma$ denotes the gamma function. $I_p(\cdot)$ is Modified Bessel function of the first kind with order $p$. $A_p(\cdot)$ is a fraction of two Modified Bessel function of the first kind.

\begin{equation}
\begin{split}
    \label{eq:A_I}
    &I_p(\kappa) = \sum_{m=0}^{\infty}\frac{1}{m! \Gamma(m+p+1)}(\frac{\kappa}{2})^{2m+p}, \\
    &A_p(\kappa) = \frac{I_{\frac{p}{2}}(\kappa)}{I_{\frac{p}{2}-1}(\kappa)}
\end{split}
\end{equation}

\subsection{Second Term Calculation}
In Eq.~\eqref{eq:kld}, the second term cannot be evaluated in closed form using the modified Bessel function $I_p(\cdot)$.
Therefore, we adopt an asymptotic approximation as given in~\cite[Eq.~(10.41.3)]{NIST:DLMF}.

\begin{equation}
\begin{split}
    \label{eq:mbf}
    I_v(vz) &\sim \frac{e^{v\eta}}{(2 \pi v)^{1/2}(1+z^2)^{1/4}} \sum_{k=0}^{\infty}\frac{U_k(x)}{v^k} \\
    \eta &= (1+z^2)^{1/2} + \log \frac{z}{1+(1+z^2)^{1/2}} \\
    x &= (1+z^2)^{-1/2} \\
    U_{k+1}(x) &= \frac{1}{2}x^2(1-x^2)U_{k}^{'}(x) + \frac{1}{8}\int_{0}^{x}(1-5t^2)U_k(t)dt \\
\end{split}
\end{equation}

Using this approximation, we compute the second term as follows, truncating the series at $k=5$.

\begin{equation}
\begin{split}
    \label{eq:second term calculation}
    &\log \frac{I_{\frac{p}{2}-1}(\kappa_1)}{I_{\frac{p}{2}-1}(\kappa_0)}
        = \log \frac{I_v(vz_1)}{I_v(vz_0)} \\
    &= v(\eta_1 - \eta_0) + \frac{1}{4}(\frac{1+z_0^2}{1+z_1^2}) + \log         
        \sum_{k=0}^{\infty}\frac{U_k(x_1)}{v^k} \cdot \frac{1}{{\sum_{k=0}^{\infty}\frac{U_k(x_0)}{v^k}}} \\
    & v = \frac{p}{2}-1
\end{split}
\end{equation}

\subsection{Stabilization of the Concentration Parameter $\kappa$}
\label{appendix:stabilize_kappa}

To address numerical instability during training, we use the approximated concentration parameter \( \kappa \) as defined in Eq.~\eqref{eq:kappa_approximation}. However, even this approximation can lead to excessively large values of \( \kappa \) as training progresses. To mitigate this, we propose two stabilization techniques.

Figure~\ref{fig:stabilize_kappa} illustrates the original approximated \( \kappa \) values computed from Eq.~\eqref{eq:kappa_approximation}, alongside the stabilized values obtained after applying our proposed techniques.

\begin{figure}[h!]
    \centering
    \includegraphics[width=1.0\linewidth]{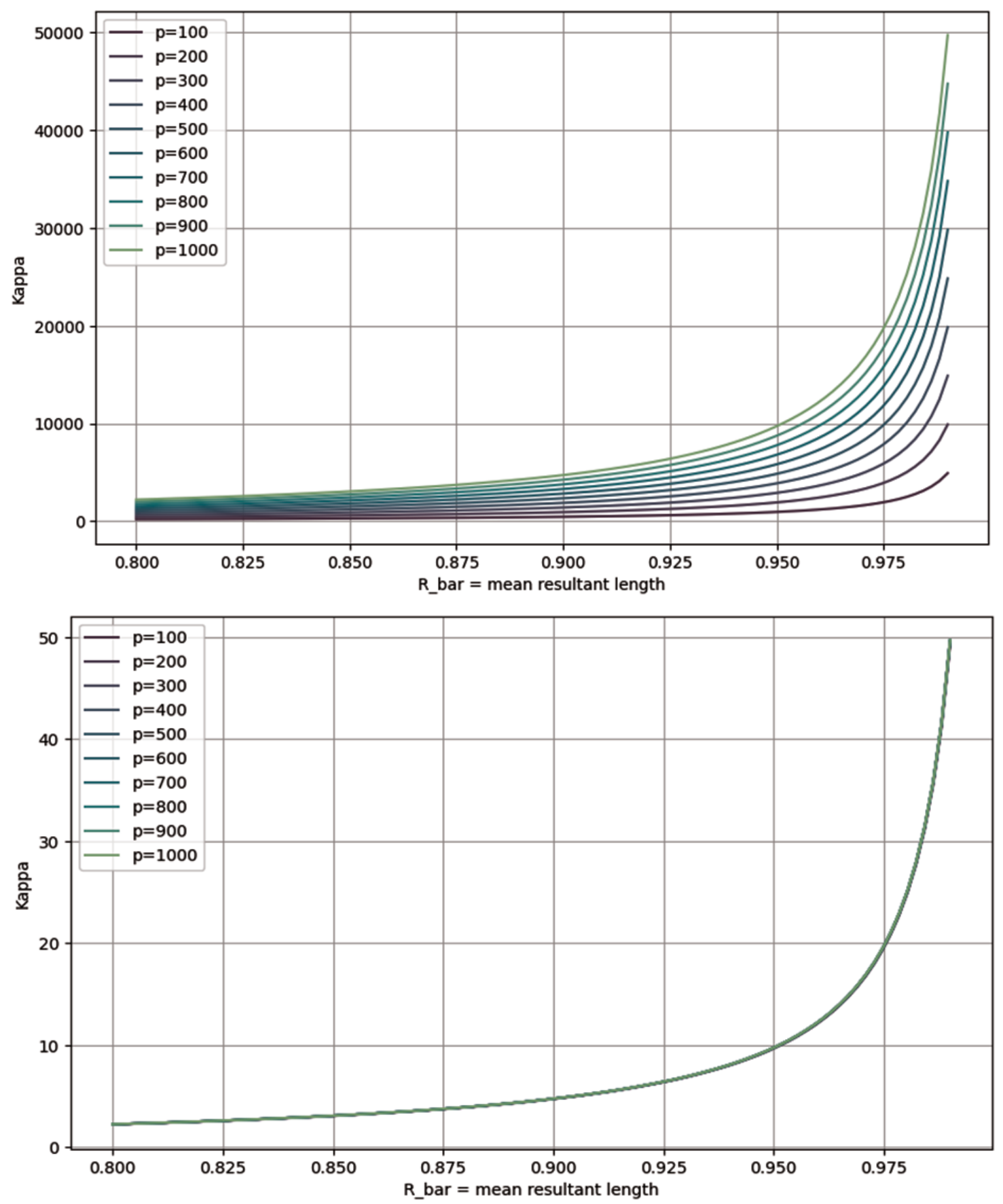}
    \caption{\textbf{Consistency by Kappa Normalization}. $\kappa$ value is calculated as approximated formulation as Eq.~\eqref{eq:kappa_approximation}. $x$-axis is $\bar{R}$ and $y$-axis is calculated $\kappa$. $p$ is feature dimension size. In our backbone ResNet18, feature dimension is 128. \textbf{(upper)} Directly calculated $\kappa$. \textbf{(lower)} $\kappa$ is divided by feature dimension $p$ and calculated with $\bar{R}$ multiplied by $\lambda_{\bar{R}}$. Without these two techniques, $\kappa$ value is going excessively large.}
    \label{fig:stabilize_kappa}
\end{figure}
%-------------------------------------------------------------------------
\section{Experimental Details}

\subsection{Training Details}

Every experiment is conducted with one or four NVIDIA GeForce RTX 3090.

To train on ImageNet-100 dataset, MoCo v3 with ViT model as the encoder from scratch. We use AdamW as our optimizer. The weight decay is 0.1 and the encoder momentum is 0.99 for updating encoder. We choose batch size of 512, learning rate of 0.03. In DSF, a larger learning rate of 0.06 is used. GPU used for overall training is NVIDIA GeForce RTX 3090.

\subsection{Transfer Learning}
We conduct transfer learning experiments following the protocol in~\cite{chen2020simple}. We report top-1 accuracy for CIFAR-10, CIFAR-100 and the Describable Textures Dataset (DTD); and mean per-class accuracy for FGVC Aircraft and Oxford-IIIT Pets.

\subsection{Robustness under Distribution Shift}
\label{appendix:transfer_learning}
We adopt the corruption methodology introduced by ~\cite{hendrycks2019benchmarkingneuralnetworkrobustness}. The ImageNet-100-C dataset is constructed by extracting the 100 selected classes from the original ImageNet-C benchmark. We evaluate the models using the mean Corruption Error (mCE) metric on corrupted images with severity level 1.

The evaluation covers all 18 corruption types provided in the benchmark:
gaussian noise, shot noise, impulse noise, defocus blur, glass blur, motion blur, zoom blur, snow, frost, brightness, contrast, elastic transform, pixelate, jpeg compression, speckle noise, gaussian blur, spatter, and saturate.

% \begin{figure}[t]
%     \centering
%     \includegraphics[width=\columnwidth]{positive_cmyk.pdf} \\
%     \caption{Positive}
%     \vspace{0.5em}
%     \includegraphics[width=\columnwidth]{negative_cmyk.pdf}
%     \caption{Negative}
%     \caption{(Top) Positive pairs. (Bottom) Negative pairs.}
%     \label{fig:positive_negative}
% \end{figure}

\begin{figure}[t]
    \centering
    % Top: Positive
    \begin{subfigure}{\columnwidth}
        \centering
        \includegraphics[width=\columnwidth]{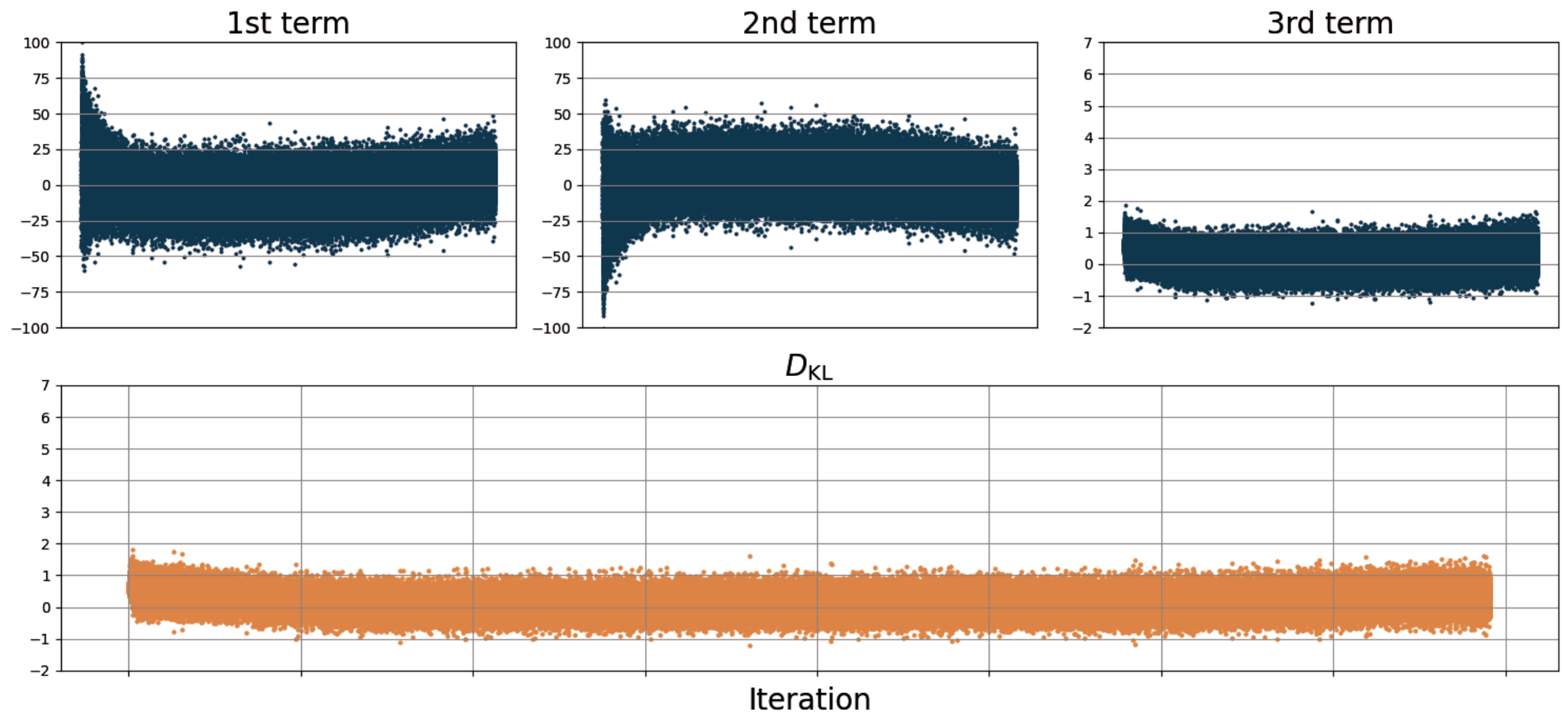}
        \caption{Three terms and KL divergence value of positive pairs.}
    \end{subfigure}
    \vspace{0.5em}
    
    % Bottom: Negative
    \begin{subfigure}{\columnwidth}
        \centering
        \includegraphics[width=\columnwidth]{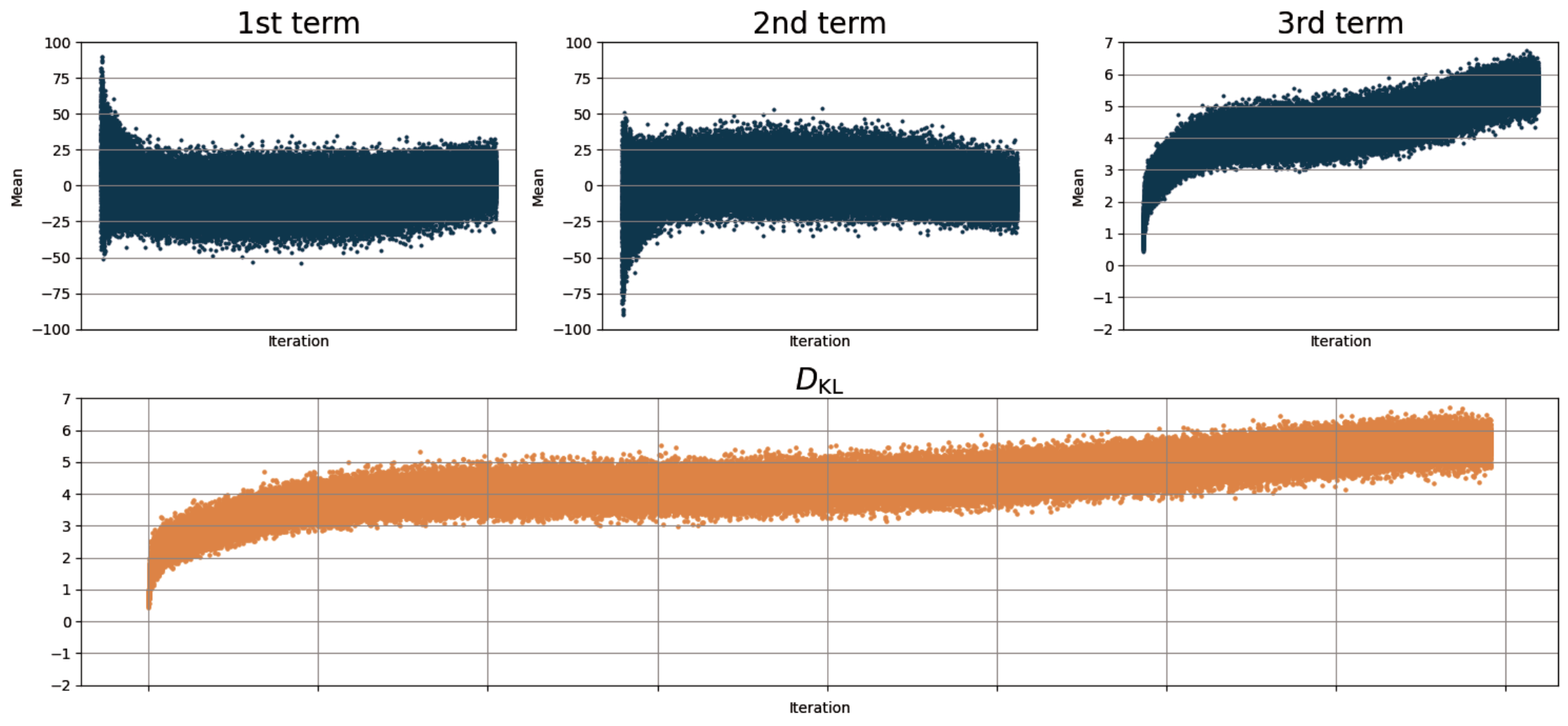}
        \caption{Three terms and KL divergence value of negative pairs.}
    \end{subfigure}
    
    \caption{(Top) Positive pairs. (Bottom) Negative pairs. Positive pairs maintain KL divergence values close to 0, reflecting their similarity, while negative pairs exhibit larger divergence values, reflecting increasing dissimilarity.}
    \label{fig:three_terms}
\end{figure}


% \bibliography{aaai2026}

\end{document}